\documentclass[11pt]{article}
\input{custom.tex}

\title{A Prior Distribution over Directed Acyclic Graphs for Sparse Bayesian Networks}
\author[2]{Felix L. Rios\thanks{flrios@kth.se}}
\author[1]{John M. Noble\thanks{noble@mimuw.edu.pl}}
\author[2]{Timo J.T. Koski\thanks{tjtkoski@kth.se}}
 \affil[1]{Institute of Applied Mathematics and Mechanics,  University of Warsaw}
 \affil[2]{KTH Royal Institute of Technology}

\begin{document}

\maketitle
\begin{abstract}
\label{abstract}
The main contribution of this article is a new prior distribution over directed acyclic graphs, which gives larger weight to sparse graphs. This distribution is intended for {\em structured} Bayesian networks, where the structure is given by an {\em ordered block model}. That is, the nodes of the graph are objects   which fall into categories (or {\em blocks}); the blocks have a natural ordering. The presence of a relationship between two objects is denoted by an arrow, from the object of lower category to the object of higher category.  The models considered here were introduced in \citet{Kemp2004} for {\em relational data} and extended to {\em multivariate data} in \citet{Mansinghka06structuredpriors}. The prior over graph structures presented here has an explicit formula. The number of nodes in each layer of the graph follow a Hoppe Ewens urn model. 

We consider the situation where the nodes of the graph represent random variables, whose joint probability distribution factorises along the DAG. We describe Monte Carlo schemes for finding the optimal aposteriori structure given a data matrix and compare the performance with \citeauthor{Mansinghka06structuredpriors} and also with the uniform prior.
\end{abstract}
\section{Introduction} 
\label{sec:introduction}
\subsection{The Block Model and Ordered Block Model} The {\em block model} and {\em ordered block model} were introduced by \citet{Kemp2004}(2004) for analysing  {\em relational data}. This is an interesting and versatile idea for the situation where classification of individuals is to be inferred from observing how the individuals relate to each other. Consider $d$ individuals, represented as nodes on a graph. For example, on a humorous note, Kemp et. al. suggest an ecclesiastical gathering where the participants are dressed informally, so that the observer cannot infer the class structure of the group simply from observing each individual separately. Each has an  unknown status; vicar, bishop, archbishop, or any other status within the Anglican hierarchy. The observer is not an ecclesiastical expert and does not know much about the hierarchical structures within Anglicanism. In particular, he does not know a priori the number of classes. He notes how these individuals interact with each other. If one individual behaves deferentially towards another, it may be assumed that he is from a class of lower order; a directed arrow is inserted from the individual of lower order to the individual of higher order, thus producing a directed acyclic graph (DAG), where each individual is represented as a node. It may be that the individuals present at the event do not all know each other; a vicar who meets a bishop will show deference only if he knows the bishop; he therefore does this with probability $p$, where $0 < p < 1$. 

To assess the class structure, one starts with a prior distribution over classifications and then updates to a posterior distribution, given the observed information of the DAG. Either there is a directed edge, or there is no edge, or else no interaction was observed and the edge status is unknown. \citeauthor{Kemp2004} propose a {\em Hoppe urn scheme} model to construct a prior distribution over the class structure. Conditioned on the class structure, they propose a Beta distribution for the edge probability between nodes in different classes, which leads to a joint prior distribution over classifications and DAGs. From this, the posterior distribution over classification, given the DAG, may be computed.

Two models are proposed; the {\em ordered block model} where there is a hierarchy and the classes have a distinct ordering and the {\em block model}, where there is not a hierarchical structure between classes. In this article, we only consider the first of these; the {\em ordered block model}. 

These models are used for wide ranging problems of relational data. Another example developed by \citeauthor{Kemp2004} is to infer the societal structure of aboriginal tribes. They also use the model for experiments in cognition, simulating the process of human learning. Children are given a collection of $d$ objects, each  identical in appearance. Some belong to category $A$, the others to category $B$, although the children are not told the number of different categories in advance. In one experiment, when an object from category $A$ touches one from category $B$, the category $B$ object may light up. There is a probability $p$ that a given category $A$ object activates a given category $B$ object. Activation is represented by a DAG, from which the classification of the objects is to be inferred. The {\em ordered block model} is relevant when objects from category $A$ do not light up; the {\em block model} is relevant when the action is reciprocal.  

\subsection{Multivariate Data and Classification of Variables} Multivariate data comes in the form of an $n \times d$ data matrix ${\bf x}$, where each row represents an independent instantiation of a random $d$-vector with probability distribution $\mathbb{P}_{X_1, \ldots, X_d}$. Throughout, ${\bf x}$ will  denote the matrix of instantiations and ${\bf X}$ to the underlying random matrix. Also, we use $X= (X_1, \ldots, X_d)$ to denote the random vector (taken as a row vector, in accordance with the way that data is presented in a data matrix).

Much work has been carried out concerning {\em classification of instantiations}, where one (or several) of the variables is a {\em class variable} and, given values taken by the other variables, the classification of the instantiation should be inferred. There is less work concerning {\em classification of variables}. 
An important contribution in this direction, where the number of classes is a priori unknown, is found in \citet{Mansinghka06structuredpriors}, who extend the work of \citeauthor{Kemp2004} to accommodate the situation where the {\em individuals} in question, which comprise the node set of the graph, are random variables.  

In many `real world' problems, variables may be grouped into classes which indicate the nature of their probabilistic relations. For example, in the QMR-DT network, introduced by \citet{Qmr91probabilisticdiagnosis}, variables fall into two classes; {\em diseases} and {\em symptoms}, where various diseases may cause various symptoms. 
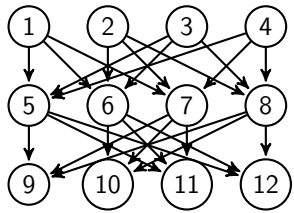
\begin{figure}[htbp]
\centering
   \begin{tikzpicture}[->,>=stealth',shorten >=1pt,auto,node distance=1.5cm,
      thick,main node/.style={circle,scale=0.7,fill=white!10,draw,font=\sffamily\Large}]
      \node[main node] (1) [] {1};
      \node[main node] (2) [right of=1] {2};
      \node[main node] (3) [right of=2] {3};
      \node[main node] (4) [right of=3] {4};
      
      \node[main node] (5) [below of=1] {5};
      \node[main node] (6) [right of=5] {6};
      \node[main node] (7) [right of=6] {7};
      \node[main node] (8) [right of=7] {8};
      
      \node[main node] (9) [below of=5] {9};
      \node[main node] (10) [right of=9] {10};
      \node[main node] (11) [right of=10] {11};
      \node[main node] (12) [right of=11] {12};
      \path[every node/.style={font=\sffamily\small}]
      (1) edge node [] {} (5)
	      edge node [] {} (6)        
	      edge node [] {} (7)        
      (2) edge node [] {} (6)
	      edge node [] {} (7)
	      edge node [] {} (8)
      (3) edge node [] {} (5)
	      edge node [] {} (6)
	      edge node [] {} (8)
      (4) edge node [] {} (5)
	      edge node [] {} (7)
	      edge node [] {} (8)
      (5) edge node [] {} (9)
	      edge node [] {} (11)        
	      edge node [] {} (12)        
      (6) edge node [] {} (10)
	      edge node [] {} (11)
	      edge node [] {} (12)
      (7) edge node [] {} (9)
	      edge node [] {} (10)
	      edge node [] {} (11)
      (8) edge node [] {} (9)
	      edge node [] {} (10)
	      edge node [] {} (12);
    \end{tikzpicture}
\caption{The HEPAR II network.}
\label{fig:heparII}
\end{figure}
In \citet{Mansinghka06structuredpriors}, the HEPAR II network, due to \citet{Onisko00learningbayesian} is considered. The DAG is shown in Figure~\ref{fig:heparII}. It is a Bayesian Network for analysing liver diseases where the variables can be divided into three classes; {\em risk factors}, {\em diseases}, {\em symptoms}. The conditional probabilities for the HEPAR II network have been learned from a database of medical cases.  This example is considered in by \citeauthor{Mansinghka06structuredpriors} when evaluating their algorithm.

Another example is the problem of finding genome pathways. Faced with $n$ instantiations (where $n$ is of the order of thousands) of the expression levels of $d$ genes (where $d$ is of order tens of thousands), the aim is to classify the genes and discover the class of {\em regulators}, which are responsible for controlling the activation of other genes, and which may influence each other.

The aim of the data analysis is therefore two-fold: to find a suitable classification of the variables and also {\em structure learning}; namely, to find a graph which encodes the direct influences that the variables have on each other, where the layers of the graph correspond to the classes of an {\em ordered block model}.

\subsection{Bayesian Networks and Structure Learning}

A collection of random variables $\{X_1, \ldots, X_d\}$ can be ordered in $d$ ways and each gives rise to a factorisation: for a permutation $\sigma$ of $\{1, \ldots, d\}$, 

\begin{equation}\label{eqfact} \mathbb{P}_{X_1, \ldots, X_d} = \prod_{j=1}^d \mathbb{P}_{X_{\sigma(j)}|\mbox{Pa}( \sigma(j))}
\end{equation}

\noindent where, with the ordering $\sigma$, for each $j$, $\mbox{Pa}(X_{\sigma(j)}) \subseteq \{X_{\sigma(1)}, \ldots, X_{\sigma(j-1)}\} =: {\cal P}_j^{\sigma}$ and ${\cal P}_j^{\sigma}$ denotes the {\em $\sigma$-predecessors} of $X_{\sigma(j)}$. For the ordering $\sigma$, for each $j$, the set $\mbox{Pa}(\sigma(j))$ is taken as the smallest subset of ${\cal P}_j^{\sigma}$ such that Equation~\eqref{eqfact} is true.  A factorisation may be expressed by a Directed Acyclic Graph (DAG, also known as acyclic digraph, ADG), a directed acyclic graph which has a directed arrow $j \rightarrow k$ if and only if $X_j \in \mbox{Pa}(k)$. Such a factorisation, together with the corresponding DAG, is known as a Bayesian Network (or BN).

For multivariate data with large numbers of variables, it is usually not feasible to learn the empirical probability distribution; if variable $X_j$ has $k_j$ possible states for $j = 1, \ldots, d$ and no assumptions may be made about the independence structure, then a total of $\prod_{j=1}^d k_j - 1$ values need to be stored. Even if each variable only has two possible states, this value is $2^d - 1$, which is growing exponentially in the number of variables. If the distribution may be factorised according to a BN where the DAG is sparse in the sense that, for each $j \in \{1, \ldots, d\}$ the set $\mbox{Pa}(j)$ is reasonably small, the storage may be reduced considerably. Furthermore, if there is a substantial independence structure that can be exploited, there are fewer parameters to be estimated from the $n$ data points and hence the estimation is more accurate. 

The key result for Bayesian Networks is that a Directed separation (D-separation) statement for the DAG implies that the corresponding conditional independence statement for the probability distribution is true. It is computationally straightforward to verify graphical separation; the whole point of expressing a probability distribution as a Bayesian Network is to use graphical separation algorithms to infer those conditional independence statements for the random variables that correspond to graphical separation statements. 

Unfortunately, in almost all practical situations, there does not exist a DAG such that the converse is true; it is not possible, in general, to find a DAG such that D-connection in the DAG implies the corresponding conditional independence statement. Therefore, the main thrust of the Bayesian Network approach is (a) to find the BN with the sparsest graph along which the probability distribution factorises and (b) conclude conditional independence when there is D-separation. Those conditional independence statements which cannot be expressed in terms of graphical separation are beyond the scope of the Bayesian Network approach.

Representing a probability distribution as a BN is often particularly efficient if the DAG represents a {\em causal} structure that is known in advance. If there are known cause-to-effect relationships between the variables, then the variables are ordered in such a way that a {\em cause} has a lower order than an {\em effect}. It is well established, though, that it is, in general, not possible to infer causality simply from data; if a DAG is learned from data, directed arrows do not in and of themselves suggest that a `cause' to `effect' relationship may be present. 

In some situations, it may happen that there exists a BN such that the set of graphical separation statements for the DAG is  equivalent to the set of conditional independence statements for the probability distribution. In such a situation, the DAG is said to be {\em faithful} to the probability distribution. Faithfulness is extremely rare, but it can occur in cases where all the variables influencing the system are observable and where there is a natural causal ordering between the variables.   

Structure learning algorithms fall broadly into two categories; {\em constraint based} and {\em search and score}. There are also hybrid algorithms which are a mixture of the two, which try to take the best features from both. While constraint based algorithms are, in general, more economical, they have well established problems. These are mostly connected with lack of faithful graph. A DAG contains an edge between $X$ and $Y$ if and only if there is no subset $S$ which $D$-separates $X$ and $Y$, but those algorithms that work on the principle of removing an edge $X \sim Y$ whenever a conditioning set $S$ is found such that $X \perp Y|S$ (practically all constraint based algorithms) are guaranteed to perform badly in  `real world' situations, since only synthetic data simulated from a distribution with a faithful graph has any serious chance of  coming from a distribution which has a faithful graph (they perform well in such situations, which usually form the basis of the criteria used for evaluation). Other minor difficulties (minor because these algorithms perform badly even if there is a perfect oracle) are connected with the power of the conditional independence tests when there is a large set of conditioning variables. Edges are wrongly removed because `do not reject' independence is wrongly taken to mean `accept' independence, hence deletion of an edge; weak tests lead to false negatives, which contradict other `reject conditional independence' statements derived from tests which are more reliable, because the statement `reject the null hypothesis' conforms to the principles of statistical theory, while taking `do not reject the null hypothesis' to mean `accept the null hypothesis' violates these principles.

Search and score methods do not have these difficulties.   A score function is defined on the space of DAGs and the aim is to find the DAG which gives the highest score. They are more rubust, but computationally substantially more expensive.  

Exhaustive search  by    scoring    every structure   is usually not computationally feasible; a good search and score algorithm will try to find favourable regions in the search space. Since not all structures are visited, there are no guarantees that the optimal graph will be chosen; the aim is to locate a reasonable structure that encapsulates the main features of the independence relations between the variables.  

One way to score the graphs is to consider a prior distribution over graph structures $\mathbb{P}_{{\cal G}}$, the Cooper-Herskovits graph likelihood function $\mathbb{P}_{{\bf X}|{\cal G}}$, which is the probability distribution over ${\bf X}$ given the graph structure.  The posterior is proportional to: $\mathbb{P}_{{\cal G}}\mathbb{P}_{{\bf X}|{\cal G}}$ and this may be used as a  score function. The structure which maximises this score maximises the posterior distribution and hence we call it the {\em maximum   aposteriori structure}   (MAPS).   The Cooper-Herskovits likelihood is well known, with a convenient closed form. To complete the score function, a prior distribution $\mathbb{P}_{{\cal G}}$ over structures has to be established. 

\subsection{Distributions over Graph Structures} The choice of prior distribution $\mathbb{P}_{{\cal G}}$ is clearly important for search-and-score based algorithms. One choice is the {\em uniform prior}; for $d$ nodes, the probability of choosing a particular DAG $G$ is:   $\mathbb{P}_{{\cal G}}(G)= a_d^{-1}$,  where $a_d$ is the well known  number of  DAGs with $d$ nodes. \citet{kuipers2013uniform}(2013) show  how to sample     a  DAG from a uniform distribution.

A straightforward way to generate a random DAG, where the distribution is not uniform, is: firstly, take a random permutation of the $d$ nodes $\sigma$, each permutation with probability $\frac{1}{d!}$. Next, randomly generate an upper triangular matrix $D$ where elements $D_{ij} = 0$ for $i \geq j$ and the other elements of the matrix are assigned the value $1$ with probability $p$ and $0$ with probability $1-p$. The graph then has directed edge $\sigma(i) \rightarrow \sigma(j)$ if and only if $D_{ij} = 1$.     

Conditioned on the permutation $\sigma$, the  probability of  obtaining  a DAG with exactly $k \leq \frac{d(d-1)}{2}$  edges is   $p^k(1-p)^{d(d-1)/2-k}$. Computing the marginal probability of a given graph structure is a difficult problem here; the number of different permutations giving rise to the same graph depends on the number of different nodes used, which depends on the number of connected components;

\[ \left\{ \begin{array}{l} \mathbb{P}_{\cal G}(G) = \frac{n(G)}{d!} p^k(1-p)^{d(d-1)/2-k} \\ \mbox{$G$ has $k$ edges; $n(G)$ is number of permutations where this graph is possible}. \end{array}\right. \]

\noindent There are various approaches to the construction of prior distributions $\mathbb{P}_{{\cal G}}$ corresponding to practical 
prior   information  about  the structure of the  DAG.  
In \citep{mukherjee2008network}, priors of the  form 

\[ \mathbb{P}_{\cal G}(G) \propto  e^{\lambda\sum_{i} w_{i} f_{i}(G)},\]

\noindent where  the functions $\{  f_{i}(G)   \} $ are called {\em concordance functions}. These functions are constructed from prior information about features that the graph is likely to possess; e.g. individual edges, classes of  vertices, sparsity and  degree distributions. 

\section{Outline of the Method and Results} 
\label{sec:outline_of_the_method_and_results}
As with \citet{Mansinghka06structuredpriors}, we consider situations where there is a natural, but unknown, partition of the {\em variables} into classes and the probability distribution over the variables may be expressed as a Bayesian Network where the DAG only has the possibility of an edge $x \rightarrow y$ if and only if $x$ belongs to a class of strictly lower order than class $y$. This is an {\em ordered block model}.  \citet{Kemp2004} introduce a prior, which we call the {\em Hoppe-Beta Prior}, which is a joint distribution over classifications and graphs. In our case, we make the further assumption that this partition of the variables corresponds to the minimal layering of the DAG of the Bayesian Network.

\paragraph{The Hoppe-Beta Prior}  We describe the {\em Hoppe-Beta} prior; this is the name we give to the prior introduced by \citet{Kemp2004}. The description also explains how to sample from the distribution. Each step is computationally straightforward and uses natural tools. The distribution has three parameters, which may be used to  control  the sparsity and {\em consistency} of the sampled graphs; consistency will be defined later.  The algorithm can be divided into two main steps: 

\paragraph{Step 1} In the first step the nodes are partitioned into numbered classes using the  Hoppe-Ewens urn scheme \citet{hoppe2009}(1984). The prior over classes is therefore constructed first, without reference to the prior over DAGs.

\paragraph{Step 2} The second step is the prior over DAGs conditioned on the classification.  Firstly, the edge probability from a class $a$ node to a class $b$ node is generated using a suitable Beta distribution, these are independent for different pairs of classes and only edges from lower to higher classes are permitted. A DAG is then generated using these probabilities; conditioned on the random variables generated by the Beta distributions, the indicator variables for edges between pairs are mutually independent. 

\paragraph{The Minimal Hoppe-Beta Prior} We introduce a new prior, over graph structures, which we call the {\em Minimal Hoppe-Beta Prior}. As with the Hoppe-Beta prior, we first generate a class vector, according to the Hoppe-Ewens urn scheme. We then ensure that this structure provides the {\em minimal layering} for a DAG. That is, a node in class $C_i$ has at least one parent in class $C_{i-1}$. We first produce a skeleton; for each node in a class $C_i$, choose one node at random (each with equal probability) from class $C_{i-1}$ and add an arrow. We call these {\em compelled edges}. This skeleton graph ensures that the classification from the Hoppe-Ewens scheme provides a minimal layering. Then we decide on the remaining edges using the same scheme as \citeauthor{Kemp2004} with the Hoppe-Beta prior. 

We are able to obtain a convenient closed form expression for our prior. It has the advantage that it is a prior only over graphs. The graph structure implies a class vector, which is the minimal class vector. If one wishes to infer class structure from a DAG, then the minimal layering represents as much of the  class structure that can be inferred from data alone. 

There are other possible choices of class structure other than {\em minimal} layering. The point is that  data influences the class structure only through the update on the graph structure. 

\paragraph{The Posterior Distribution}  The nodes of the graph are random variables $(X_1, \ldots, X_d)$. Given an $n \times d$ data matrix ${\bf x}$ of $n$ instantiations, the Cooper-Herskovits likelihood is used to give a likelihood function for the graph structure given data. Given the data matrix, this likelihood only depends on the graph structure; the data influences classification only through the graph structure. This is true both for the Hoppe-Beta prior joint distribution over class /graph structure (where the distribution over classes given the graph does not change with data) and for the Minimal Hoppe-Beta prior over graph structure (where data updates the distribution over graphs, and the minimal layering is chosen for the class structure).  

\paragraph{Monte Carlo Methods for the Posterior} A Gibbs sampler is considered for the posterior distribution, but this turns out to be rather slow and wrongly classified nodes have difficulties changing class. To find the Maximum Aposteriori Structure (MAPS), a stochastic optimisation algorithm is used. The moves are based on the Hoppe-Ewens urn scheme for moving between classes, along with some addition / deletion of edges. A proposed move $x \mapsto y$ is accepted with probability $\min\left (1, \frac{\mathbb{S}(y)}{\mathbb{S}(x)} \right )$ where $\mathbb{S}$ is the score function.   

We compare inference from the posterior for three prior distributions: the uniform prior over graph structures, the Hoppe-Beta distribution where we consider the graph structure and the Minimal Hoppe-Beta prior.

\section{Block Structured Priors}
We now describe the   Hoppe-Beta  distribution over classification and graph structures.

Let $X = (X_1, \ldots, X_d)$ denote $d$ nodes of a graph. The indexing set is $(1, \ldots, d)$. Each node belongs to a {\em class}, where a priori the assignment of nodes to classes is unknown and the number of classes is also unknown. Let $z=(z_1,z_2,\dots,z_d)$ be the {\em class assignment vector}; , where $z_i=j$ means that variable $i$ is of class $j$; the classes are labelled by the positive integers.

For the {\em ordered block model}, a DAG represents direct influences between the nodes. The nodes are also classified and all the arrows of the DAG are {\em from} nodes of a lower class {\em to} nodes of a higher class. For the prior distribution, firstly a classification vector is generated  via a Hoppe-Ewens urn scheme and then a DAG is generated based on the class structure. The classes represent a hierarchical structure. Therefore, the classification vectors $z^{(1)} = (1,2,2)=\{\{1\}_1,\{ 2,3\}_2\}$ and $z^{(2)}=(2,1,1)=\{\{ 2,3\}_1,\{1\}_2\}$ give {\em different} classification structures; for the DAG, directed arrows will go {\em from} variables of a class of lower index {\em to} variables of a class of higher index.
We shall call the sets in a partition cells, layers, classes or colours.
We denote a permutation of $d$ elements in the Cauchy 2-line notation.
E.g. the permutation $\rho$ defined by $\rho(1)=2,\,\rho(2)=3,\,\rho(3)=4,\,\rho(4)=1,$ is denoted by $\bigl(\begin{smallmatrix}
  1 & 2 & 3 & 4  \\
  2 & 3 & 4 & 1
\end{smallmatrix}\bigr)$.

\subsection{Prior Distribution over Node Classification}  
 Let $z =(z_1, \ldots, z_{d})$ denote a class assignment vector, generated   by a Hoppe urn model. The nodes $\{1, \ldots, d\}$ are introduced one by one, in order lowest to highest.  An urn initially has an orange ball of weight $\alpha > 0$. 

The size of $\alpha$ will influence the number of classes; the smaller $\alpha$ the fewer classes.  If $\alpha = 0$, then all the variables will be in a single class and hence the resulting DAG will be the empty DAG. 

Each ball added to the urn has unit weight. At the $n$th selection, we  draw  a ball at  random, in proportion to its weight, from the urn.  
If we draw the   orange   ball,  then we put it back,  together with an additional  ball of a colour that has not yet present in the urn. This new colour is the `value' of $z_n$. The colours are numbered according to the order in which they were introduced to the urn.   If we do not pick an orange ball, we  put the  ball back, together   with another ball of the same colour and, in this case, this is the value of  $z_n$.   

The orange ball takes the label $0$; note that none of the items introduced are placed in class $0$. 
  
Firstly, node $1$ is assigned to class $1$ (a colour different from orange). 

Assume that nodes $1, \ldots, j$ have been assigned to classes and that there are now a total of  $k_j$ colours different from orange.    
  For $(z_1,\dots, z_d)\in \{1,\dots,d\}^d$ we set  
  \begin{align}
    m_i^{(j)}(z_1,\dots, z_d) = \sum^j_{l=1}\mathbf 1_i(z_l),\,i=1,\dots,d,  \nonumber
  \end{align}
  where 
  \begin{align}
    \mathbf 1_i(x) = \left\{ 
      \begin{array}{l l}
        1, & x=i \\
        0, & x\ne i.
      \end{array} \right. \nonumber 
  \end{align}
  Set 
  \begin{align} 
    m_i^{z;j} \stackrel{\rm def}{=}  m_i^{(j)}(z_1,\dots, z_d).
  \end{align} 
  Thus  $m_i^{z;j} $  counts  the number of nodes from $\{1,\dots, j\}$ in cell $i$, so that $\sum_{i=1}^{k_j}m_i^{z;j} =j$.  

\paragraph{Generation} The classification vector $(z_1, \ldots, z_d)$ is considered as the outcome of a random vector $(Z_1, \ldots, Z_d)$. The algorithm begins with $z_1 = 1$ and proceeds as follows:  For $j\in \{1,\dots,d-1\}$,  

\begin{equation} \left\{ 
\begin{array}{ll}
\mathbb{P}_{Z_{j+1} |Z_1, \ldots, Z_j}(k_j+1 |z_1, \ldots, z_j) = \frac{\alpha}{\alpha+j} & \\
      \mathbb{P}_{Z_{j+1} |Z_1, \ldots, Z_j }(i |z_1, \ldots, z_j ) = \frac{m_i^{z:j}}{\alpha+j} & i=1,\dots, k_j.
      \end{array} \right. 
    \label{eq:crp_prob}
  \end{equation}

\noindent Let
\[ K:= \max_{j \in \{1, \ldots, d\}} z_j.\]

\noindent so that $K$ is the total number of classes. 
\subsubsection*{Expected number of cells}
 
Let $K_d$ be the number of cells generated by Hoppe's urn scheme  with $d$ balls. Let $I_i$ to be the  indicator function of the event that a new class is created in round $i$, $i=1,\dots, d$.
 Then $\mathbb{P}(I_i=1)=\mathbb{E}[I_i]=\frac{\alpha}{\alpha + i -1}$ for $i=1,\dots, N$. The expected number of classes is:
 
\begin{equation}
\label{eq:exp_cells}
\mathbb{E}[K_d] = \mathbb{E} [\sum_{i=1}^{d}I_i] = \frac{\alpha}{\alpha} + \frac{\alpha}{\alpha+1} +\dots + \frac{\alpha}{\alpha+d-1} .
\end{equation}

\noindent Note that 

\begin{equation}
\frac{\alpha}{\alpha} + \frac{\alpha}{\alpha+1} +\dots + \frac{\alpha}{\alpha+d-1} = \alpha\mathcal H_d (\alpha)
\end{equation}

\noindent where $\mathcal H_d$ is the generalised harmonic number.
It is straightforward to compute that for any fixed $\alpha$, $\mathcal H_d(\alpha) \sim \ln(d)$ as $d \to \infty$.
Thus $$\lim_{d \rightarrow +\infty}\frac{\mathbb{E}[K_d]}{\ln (d)} = \alpha.$$  

\subsection{Prior over Graph Structure, Given Class Structure} 
 
For {\em unlabelled} classes, the distribution of $Z$ is {\em exchangeable}. The same effect is achieved by randomising the order of the classes. Let $R$ be the space of possible class permutations; if there are $K$ classes, then $\rho$ is a permutation of $\{1,\dots, K\}$. The conditional distribution of $\rho$, conditioned on $Z$ is:

\begin{equation}
\mathbb{P}_{R|Z}(\rho|z)= \frac{1}{K!},\qquad \rho \in R.
    \label{eq:cluster_perm_prob}
  \end{equation}
 When constructing the DAG, we only permit edges from nodes in a class $a$ to nodes in a class $b$ if $\rho(a) < \rho(b)$. Conditioned on the classification vector $z$ and the class ordering $\rho$, edges are mutually independent of each other. Firstly, we randomly generate the edge probability and then, for each pair, the existence of an edge is the outcome of a Bernoulli random variable.

For $(a,b) \in \{1, \ldots, K\}^2$, define the density  $f_{a,b}$  as follows: 

\begin{equation} 
 \left\{ \begin{array}{ll}   f_{a,b}(x) = 
    \left\{ \begin{array}{ll}
      \frac{1}{B(\beta_{a,b;1},\beta_{a,b;2})}x^{\beta_{a,b;1}-1}(1-x)^{\beta_{a,b;2}-1}, & 0 \le x \le 1\\
      0, & {\rm otherwise}. 
    \end{array}\right. & a < b \\
    \delta_0(x) & a  \geq b
    \label{eq:edge_prob}
  \end{array} \right. \end{equation}
   
\noindent where for each $a < b$, $\beta_{a,b;1} > 0$ and $\beta_{a,b;2} > 0$. If $\beta_{a,b;1} = 0$ then $f_{a,b} = \delta_0$ and if $\beta_{a,b;2} = 0$ then $f_{a,b} = \delta_1$. Here  $B(\cdot, \cdot)$ is the  {\em beta function} defined by:

\[ B(x,y)=\dfrac{\Gamma(x)\,\Gamma(y)}{\Gamma(x+y)},\] 

\noindent where $\Gamma(\cdot)$  is the Euler gamma function.  
  
  Let $\eta=(\eta_{a,b}:a,b\in \{1,\dots,K\})$ be a random matrix of edges probabilities where $\eta_{a,b}$ are independent random variables and, for each $(a,b) \in \{1, \ldots, K\}^2$, $\eta_{a,b}\sim f_{a,b}$.
Let $\Xi$ denote the $d \times d$ matrix with entries $\xi_{i,j}=\eta_{\rho(z_i),\rho(z_j)}$.

 Let $G$ denote the edge set of the graph. This is the $d \times d$ matrix with  entries

\[ G_{ij} = \left\{ \begin{array}{ll} 1 & \mbox{edge $i \rightarrow j$ present} \\ 0 & \mbox{edge $i \rightarrow j$ not present} \end{array}\right. \]

\noindent Conditioned on the matrix $\Xi$, $G$ is the outcome of a random matrix ${\cal G}$ where the entries of ${\cal G}$ are independent and ${\cal G}_{ij} \sim \mbox{Be}(\xi_{ij})$ (i.e. a Bernoulli trial with success probability $\xi_{ij}$).

Then the (prior)  probability of  the DAG $G$  given a partition $z$, class permutation $\rho$ and the   edge probabilities $\xi$ is found    by:

\begin{equation}
 \mathbb{P}_{{\cal G} |\Xi, Z, R}(G| \xi,  z, \rho) = \mathbb{P}_{{\cal G} |\Xi}( G | \xi) =  \prod_{x=1}^{d}\prod_{y=1}^{d}\xi_{x,y}^{G_{x,y}}(1-\xi_{x,y})^{1-G_{x,y}}.
  \label{eq:d|xi,z}
\end{equation}

\noindent  with the convention $0^0 = 1$. It is clear from the construction  that $d$ is the edge set of a DAG. Note that we take $0^{0}=1$.

 \citet{Mansinghka06structuredpriors} and \citet{Kemp2004} restrict attention to the situation where $\beta_{a,b;1} = \beta_{1}$ and 
$\beta_{a,b;2} = \beta_{2}$ for all $1 \leq a < b$.  The nature of the graph prior is therefore controlled by three parameters, $\alpha$,   $ \beta_{1}$ and 
$\beta_{2}$. 

While leading to computational convenience, removing the dependence of $\beta_1$ and $\beta_2$ on $a$ and $b$ means that edges are equally likely between a node and any other of a higher order. The expected value is $\frac{\beta_1}{\beta_1 + \beta_2}$ and the variance $\frac{\beta_1\beta_2}{(\beta_1 + \beta_2)(1 + \beta_1 + \beta_2)}$. A lower $\frac{\beta_1}{\beta_1 + \beta_2}$ leads to a sparser graph.\vspace{5mm}

\noindent Another situation of interest, which is computationally convenient, but which is only valid when it is known a priori that class $i$ only influences higher order classes via class $i+1$, is the model where

\begin{equation}
    f_{a,b}(x) = \delta_0(x)  \qquad b \neq a+1
    \label{eq:delta}
  \end{equation}

\noindent and $\beta_{a,a+1;1} = \beta_1$, $\beta_{a,a+1;2} = \beta_2$. In this situation, each class has at most one adjacent class; for variables in a given class edges from these variables are only possible to variables in the adjacent class. 

A computationally convenient setting, which gives flexibility, is to consider two setting: $\beta_{i;j,j+1} = \beta_{i,1}$ for $i = 1,2$ and $\beta_{i;j,j+k} = \beta_{i,2}$ for $k \geq j+2$, $i = 1,2$. The parameters are chosen so that the edge probability from $j$ to $j+1$ is higher than $j$ to $j+k$ where $k \geq 2$. 

\begin{Ex}
\label{ex:algrun}
\end{Ex}
\noindent In \autoref{fig:nodes}-\autoref{fig:dag} a possible outcome of the algorithm is shown in 4 steps. Here $d=8$,
\[ z= ( \begin{smallmatrix} 1 & 1 & 2 & 3 & 3 & 3 & 2 & 1  \end{smallmatrix} ), \quad \rho= \left ( \begin{smallmatrix} 1 & 2 & 3 \\ 1&3&2 \end{smallmatrix} \right) .\] 
\autoref{fig:nodes} shows the nodes. 
\autoref{fig:partition} shows the partition of the nodes generated by Hoppe's urn scheme, where the colours indicate the different cells in the partition. Cell 1 is coloured in red, cell 2 is coloured in black and cell 3 is coloured in blue.
\autoref{fig:cell_order}, shows the new order of the cells, $\rho$. 
\autoref{fig:dag} shows a possible graph under these conditions.

\begin{figure}[!ht]
\centering

\begin{subfigure}[b]{0.4\textwidth}
\centering
\begin{tikzpicture}[->,>=stealth',shorten >=1pt,auto,node distance=1.0cm,
      thick,main node/.style={circle,scale=0.7,fill=white!10,draw,font=\sffamily\Large}]
      \node[main node] (1) [] {1};
      \node[main node] (2) [right of=1] {2};
      \node[main node] (3) [right of=2] {3};      
      \node[main node] (4) [right of=3] {4};
      \node[main node] (5) [right of=4] {5};
      \node[main node] (6) [right of=5] {6};      
      \node[main node] (7) [right of=6] {7};
      \node[main node] (8) [right of=7] {8};
    \end{tikzpicture}
    \caption{}
    \label{fig:nodes} 
  \end{subfigure}
  ~

  \begin{subfigure}[b]{0.3\textwidth}
    \centering
    \begin{tikzpicture}[->,>=stealth',shorten >=1pt,auto,node distance=1.5cm,
      thick,main node/.style={circle,scale=0.7,fill=white!10,draw,font=\sffamily\Large}]
      
      \node[main node,draw=red] (1) [] {1};
      \node[main node,draw=red] (2) [right of=1] {2};
      \node[main node,draw=red] (3) [right of=2] {8};
      
       \node[main node,draw=black] (4) [below of=1] {3};
      \node[main node,draw=black] (5) [right of=4] {7};

      \node[main node,draw=blue] (6) [below of=4] {4};
      \node[main node,draw=blue] (7) [right of=6] {5};
      \node[main node,draw=blue] (8) [right of=7] {6};
    \end{tikzpicture}
    \caption{ }
    \label{fig:partition} 
  \end{subfigure}
  ~
  \begin{subfigure}[b]{0.3\textwidth}
    \centering
    \begin{tikzpicture}[->,>=stealth',shorten >=1pt,auto,node distance=1.5cm,
      thick,main node/.style={circle,scale=0.7,fill=white!10,draw,font=\sffamily\Large}]
      \node[main node,draw=red] (1) [] {1};
      \node[main node,draw=red] (2) [right of=1] {2};
      \node[main node,draw=red] (3) [right of=2] {8};
      
      \node[main node,draw=blue] (4) [below of=1] {4};
      \node[main node,draw=blue] (5) [right of=4] {5};
      \node[main node,draw=blue] (6) [right of=5] {6};
      
      \node[main node,draw=black] (7) [below of=4] {3};
      \node[main node,draw=black] (8) [right of=7] {7};
    \end{tikzpicture}
    \caption{}
    \label{fig:cell_order}
  \end{subfigure}
  ~
  \begin{subfigure}[b]{0.3\textwidth}
    \centering
    \begin{tikzpicture}[->,>=stealth',shorten >=1pt,auto,node distance=1.5cm,
      thick,main node/.style={circle,scale=0.7,fill=white!10,draw,font=\sffamily\Large}]
      \node[main node,draw=red] (1) [] {1};
      \node[main node,draw=red] (2) [right of=1] {2};
      \node[main node,draw=red] (3) [right of=2] {8};
      
      \node[main node,draw=blue] (4) [below of=1] {4};
      \node[main node,draw=blue] (5) [right of=4] {5};
      \node[main node,draw=blue] (6) [right of=5] {6};
      
      \node[main node,draw=black] (7) [below of=4] {3};
      \node[main node,draw=black] (8) [right of=7] {7};
      \path[every node/.style={font=\sffamily\small}]
      (1) edge node [] {} (4)
      edge node [] {} (5)        
      (2) edge node [] {} (6)
      edge node [] {} (7)
      (4) edge node [] {} (8);
    \end{tikzpicture}
    \caption{}
    \label{fig:dag}    
  \end{subfigure}

\caption{Figures a)-d) show the steps in Example~\ref{ex:algrun}}
\label{fig:alg}  
\end{figure}
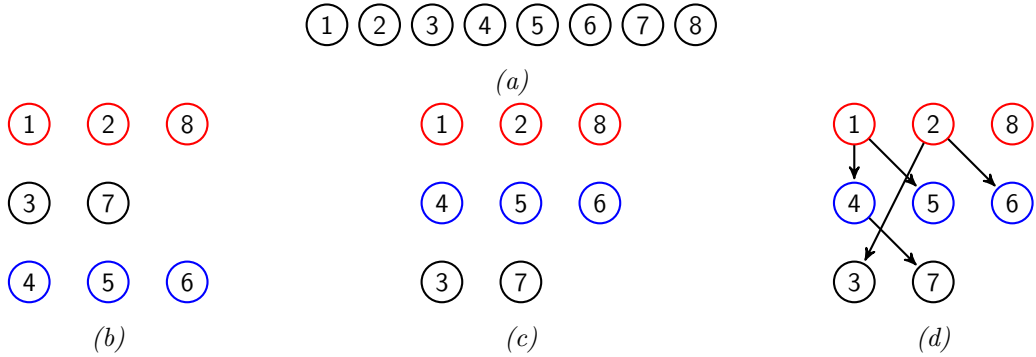

\subsection{Joint distribution of the DAG and the  partition}
The method for generating DAGs described above enables us to derive the joint distribution of the classification / DAG $(Z, {\cal G})$, where $Z$ is the random classification vector and ${\cal G}$ the random DAG.

\subsubsection{Probability of the Partition}
The probability function of a partition  $Z$ is given by the following well known theorem found in \citep{hoppe2009}. 
We need to define the space of possible $z$ vectors:
\begin{align}
  \label{eq:feasible_partitions}
  \mathcal S_d = \{z \in \{1,\dots ,d\}^d | z_1=1, 1 \le z_j \le \max_{1\le k \le j-1} z_k +1 : 2\le j \le d\}.
\end{align}
We set 
\begin{align} 
m_k  \stackrel{\rm def}{=} m_{k}^{z:d},  
\end{align}
which   is the occupation number for cell $k$ after all $d$ nodes have  been assigned to cells. 
Then we have the following result.

\begin{Th} \label{hoppethm2}
Let $K$ be the number of non-empty cells and

\begin{equation}
  \label{eqfeasiblepartitions}
  \mathcal S_d = \left \{z \in \{1,\dots ,d\}^d | z_1=1, 1 \le z_j \le \max_{1\le k \le j-1} z_k +1 : 2\le j \le d \right \}.
\end{equation}

\noindent that is $\mathcal S_d$ is the set of possible values of  $z$. 
The distribution of the random vector $Z$ for a given parameter value $\alpha \in \mathbb{R}_+$ may be computed explicitly and is given by 

\begin{equation}
\mathbb{P}_{Z }(z| \alpha) =  
\begin{cases}
\alpha^K  \frac{\Gamma(\alpha)}{\Gamma (d+\alpha)} \prod_{k=1}^K(m_k-1)!, & K=1,\dots ,d,  z\in \mathcal S_d\\
0, & \mbox{otherwise}. \\
\end{cases} 
\end{equation}
\end{Th}

\begin{proof}
Consider the probability of a sequence   $z \in \mathcal S_d$ (so that $z_{1}=1$) with  occupation numbers $m_{1}, \ldots, m_K $; the number of classes is  $K$.  The probability factorises as:

\[\mathbb{P}_{\alpha; Z }(  z ) =  \prod_{j=1}^{d-1}\mathbb{P}_{\alpha,Z_{j+1} | Z_{1},\dots, Z_{j}}(z_{j}  | 1,z_{2},\dots, z_{j} ).  \]

\noindent   From Equation~\eqref{eq:crp_prob}, it follows that

\[ \mathbb{P}_{\alpha;Z}(z) = \frac{\alpha^K \prod_{i=1}^K \prod_{l=1}^{m_i-1}l}{\prod_{j=0}^{d-1} (\alpha + j)} = \frac{\alpha^K \prod_{i=1}^K (m_i - 1)!}{\prod_{j=0}^{d-1} (\alpha + j)} \]

\noindent where  $\prod_{l=1}^{m_i-1} l = 1$ if $m_i = 1$. Finally using $\Gamma(\beta + 1) = \beta \Gamma(\beta)$, it follows that $\Gamma(\alpha + d) = \Gamma(\alpha) \prod_{j=0}^{d-1} (\alpha + j)$ from which the result follows.
\end{proof}

\subsubsection{The Marginal Prior Probability of the DAG given the  Partition}

Marginalising over $\Xi$ in $ \mathbb{P}_{{\cal G},\Xi|R,Z} = \pi_{\Xi|  Z,R}\mathbb{P}_{{\cal G}|\Xi,  Z}$ gives the following:

\begin{Th}
\begin{equation}
  \mathbb{P}_{{\cal G}|R, Z}(G|\rho,  z) = 
  \prod_{1\le \rho(a)< \rho(b)\le K} \frac{ B(\beta_{1;\rho(a),\rho(b)}+N_{a,b},\beta_{2;\rho(a),\rho(b)} +M_{a,b} )}{B(\beta_{1;\rho(a),\rho(b)},\beta_{2;\rho(a),\rho(b)} )},
  \label{eqpdrz}
\end{equation}

\noindent where $N_{a,b}$ is the number of edges between cell $a$ and cell $b$ and $M_{a,b}$ is the corresponding number of missing edges.
\end{Th}

\begin{proof}
The details of the computation  are  given below:

\begin{eqnarray*}
\lefteqn{ \mathbb{P}_{{\cal G}|R, Z}(G|\rho, z)}\\&& =  \int_{\xi}\mathbb{P}_{{\cal G}|\Xi,  Z}(G|\xi, z) \pi_{\Xi|R, Z}(\xi|\rho, z)d\xi \nonumber\\&& 
  =   \int_{\xi}\prod_{x=1}^{N}\prod_{y=1}^{N}\xi_{x,y}^{d_{x,y}}(1-\xi_{x,y})^{1-G_{x,y}} \pi_{\Xi|R, Z}(\xi|\rho, z)d \xi \nonumber\\ 
 && =  \prod_{1\leq \rho(a) < \rho(b) \leq K}\int_0^1\eta_{\rho(a),\rho(b)}^{n_{a,b}}(1-\eta_{\rho(a),\rho(b)})^{M_{a,b}}  f_{\rho(a),\rho(b)}(\eta_{\rho(a),\rho(b)}) d\eta_{\rho(a),\rho(b)} \nonumber\\ 
  && = \prod_{1\leq \rho(a) < \rho(b) \leq K}\int_0^1\eta_{\rho(a),\rho(b)}^{N_{a,b}}(1-\eta_{\rho(a),\rho(b)})^{M_{a,b}}  \frac{\eta_{\rho(a),\rho(b)}^{\beta_{1:\rho(a),\rho(b)}-1}(1-\eta_{\rho(a),\rho(b)})^{\beta_{2:\rho(a),\rho(b)}-1}}{B(\beta_{1:\rho(a),\rho(b)},\beta_{2:\rho(a),\rho(b)} )} d\eta_{\rho(a),\rho(b)} \nonumber\\ 
  && =  \prod_{1\leq \rho(a) < \rho(b) \leq K}\int_0^1 \frac{\eta_{\rho(a) ,\rho(b)}^{\beta_{1:\rho(a),\rho(b)}-1 +N_{a,b}}(1-\eta_{\rho(a),\rho(b)})^{\beta_{2:\rho(a),\rho(b)} -1+M_{a,b}}}{B(\beta_{1;\rho(a),\rho(b)},\beta_{2;\rho(a),\rho(b)} )}   d\eta_{\rho(a),\rho(b)} \nonumber\\ 
  && =  \prod_{1\le \rho(a) <\rho(b)\le K}\frac{B(\beta_{1;\rho(a),\rho(b)} +N_{a,b},\beta_{2;\rho(a),\rho(b)} + M_{a,b} )}{B(\beta_{1;\rho(a),\rho(b)},\beta_{2;\rho(a),\rho(b)} )}, \nonumber
\end{eqnarray*}

\noindent where, $N_{a,b}$ is the number of existing edges between class $a$ and class $b$ (provided $\rho(a) < \rho(b)$)   and $M_{a,b} = m_am_b - N_{a,b}$ is the number of missing edges. $m_a$ and $m_b$ denote the numbers of nodes in classes $a$ and $b$ respectively. 
Recall that $\xi_{i,j}=\eta_{\rho(z_i),\rho(z_j)}$.
\end{proof}

\noindent At this point, the {\em Minimal Hoppe-Beta Prior} (which we derive later) has a distinct advantage over the Hoppe-Beta prior of Mansingkha. To marginalise over $\rho$, some additional assumptions are needed. If we set $\beta_{i} = \beta_{i;a,b}$ for all $a < b$, $i = 1,2$, the right hand side of Equation~\eqref{eqpdrz} is the same for any valid partition of the cells for which the DAG is possible. The conditional marginal probability of a DAG $G$ given a partition $z$ is therefore given by 

\begin{equation}
  \mathbb{P}_{{\cal G}|Z}(G| z) = \frac{g(G, z)}{K!}\mathbb{P}_{{\cal G}|R,Z}(G|\rho_1, z),
  \label{eq:P(D|Z)}
\end{equation}

\noindent where $g(G, z)$ is the number of permutations of the cells that are compatible with  the DAG $G$ and $\rho_1$ is one such permutation. This is self evident: 

\begin{eqnarray*}
\mathbb{P}_{{\cal G}|Z}(G|  z) &= & \sum_{\rho \in R}\mathbb{P}_{{\cal G}|R,Z}(G|\rho, z)\mathbb{P}_{R|{ Z}}(\rho| z) \\
&  = & \sum_{i=1} ^{K!}\mathbb{P}_{{\cal G}|R,{ Z}}(G|\rho_i, z)\frac{1}{K!}  = \frac{g(G,z)}{K!}\mathbb{P}_{{\cal G}|R,Z}(G|\rho_1,  z).\nonumber
\end{eqnarray*}
 
\noindent  The quantity $g(G,z)$ is the number of topological orderings of the nodes in $\widehat G$ and is discussed, along with an algorithm for computing it,  in \citet{Wing-Ning}.\vspace{5mm}

\noindent Now, with slight abuse of notation, let $N_{a,b} = N_{\rho(a),\rho(b)}$ and $M_{a,b} = M_{\rho(a),\rho(b)}$; i.e. $N$ and $M$ denote the numbers of included and missing edges between the classes after they have been listed in order $\rho(1), \ldots, \rho(K)$. Putting these marginalizations together, it follows that for    $\beta_{1:a,b} = \beta_1$ and $\beta_{2:a,b} = \beta_2$, 

\begin{equation}
  \mathbb{P}_{{\cal G}, Z}(G, z) =   \frac{g(G, z)}{K!} \prod_{1\le a < b \le K } \frac{B(\beta_1+N_{a,b},\beta_2 +M_{a,b} )}{B(\beta_1,\beta_2 )}    \alpha^K  \frac{\Gamma(\alpha)}{\Gamma (d+\alpha)} \prod_{k=1}^K(m_k-1)! 
  \label{eqdzjoint}
\end{equation}

\noindent where $g(G,z)$ is the number of orderings $\rho$ of the classes $1, \ldots, K$ that are compatible with the DAG $G$. 

\paragraph{Note} Consider the DAG in Figure~\ref{figdag1}. Suppose the partition is $C_1 = \{1\}$, $C_2 = \{2\}$, $C_3 =\{3 \}$. If $\beta_{i;a,b} = \beta_i$: $i = 1,2$ for all $a < b$, then the edge probability $1 \mapsto 3$ is the same whether the classes appear in the order $C_1,C_2,C_3$ or $C_2, C_1, C_3$. If $\beta_{i:1,3} \neq \beta_{i:2,3}$ then the order makes a difference. 
\section{The Minimal Hoppe-Beta Prior over Graph Structures} We now present a new prior distribution over graph structures, which we call the {\em Minimal Hoppe-Beta Prior}  This is based on the Hoppe-Beta prior of \citeauthor{Kemp2004} and \citeauthor{Mansinghka06structuredpriors}, but has some features that are more convenient.	

\subsection{Hierarchical Graph Drawings} A {\em hierarchical graph drawing} or {\em layering} of a DAG ${\cal G} = (V, D)$ is a partition of the nodes in numbered layers such that all edges are directed {\em from} nodes in layers of lower rank {\em to} nodes in layers of higher rank,  c.f.~\citet{graph_drawings}. This is  a style  of graph drawing  for visual understanding of hierarchical relations. There are several algorithms  for achieving   this,   as surveyed in  \citet{Tamassia2008}.        

The layering of a DAG is not necessarily unique. Consider the DAG in Figure~\ref{figdag1}. There are 3 possibilities for layering this DAG:

 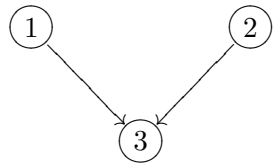
\begin{figure}[ht]
\begin{center}
 \[ \UseTips \xymatrix{  *++[o][F]{1} \ar [dr]  & & *++[o][F]{2} \ar[dl] \\ &   *++[o][F]{3} & }   \]
\caption{DAG with three nodes}
\label{figdag1}
\end{center}
\end{figure}

\begin{itemize}
 \item Layer 1: nodes 1 and 2, Layer 2: node 3.
 \item Layer 1: node 1, Layer 2: node 2, Layer 3: node 3.
 \item Layer 1: node 2, Layer 2: node 1, Layer 3: node 3.
\end{itemize}

\noindent While node 3 is always in the highest layer, there is some ambiguity with nodes 1 and 2. They can be either in separate layers or in the same layer. 

When inferring classification, the DAG structure alone cannot distinguish between these possibilities. For our prior over graph structure, we assume that the layering is {\em minimal}. For the 3-variable DAG of Figure~\ref{figdag1}, the minimal layering is the first of the possibilities listed. 

We call the layering {\em minimal} if 
\begin{itemize} \item  it has the smallest possible number of layers for the DAG under consideration and 
 \item among layerings satisfying this criterion, the nodes are in layers with as low a rank as possible. 
\end{itemize}

\noindent  The {\em minimal} layering represents, in some sense, the  class structure that can be inferred from the DAG. 

\subsection{Outline of the Minimal Hoppe-Beta Prior} 
\label{sec:outline}
\noindent For the Minimal Hoppe-Beta Prior, we use the Hoppe-Ewens urn scheme to generate the classification, just as \citet{Kemp2004}. The difference is that this classification corresponds to the {\em minimal layering} of a DAG and we choose a minimal number of edges at random to form a {\em skeleton}; namely a DAG with a minimal number of edges to ensure that the classification is a minimal classification for the DAG. We then add in additional edges randomly according to the   approach of Kemp et. al.

\begin{itemize}
 \item {\bf Step 1} For nodes $1, \ldots, d$, generate a class assignment vector according to the Hoppe-Ewens urn scheme. Using $m_i^{(j)} = \sum_{k=1}^j {\bf 1}_i(z_k)$ and let $K_j = \max\{i : m_i^{(j)} \neq 0\}$
 \[ \mathbb{P}(Z_j = i|Z_1, \ldots, Z_{j-1}) = \left\{ \begin{array}{ll} \frac{m_i^{(j-1)}}{j-1+\alpha} & i = 1, \ldots, K_{j-1} \\ \frac{\alpha}{j-1+\alpha} & i = K_{j-1} + 1. \end{array}\right. \]
 \item {\bf Step 2} Let $K = K_d$. This is the total number of classes. Let $\rho$ be a randomly chosen permutation of $(1, \ldots, K)$; conditioned on $K$ classes, each $\rho$ is chosen with probability $\frac{1}{K!}$. 
 The permutation represents the {\em ordering} of the layers, from lowest to highest.
 \item {\bf Step 3} Let $K$ be the number of classes. For each $j = 2, \ldots, K$, for each node $v \in C_{\rho(j)}$ (where $C_i$ denotes class $i$) choose a node $w$ randomly from those in class $C_{\rho(j-1)}$ (each with equal probability). Add in the directed edge $w \rightarrow v$.

 After stage 3, a skeleton graph has been produced. This is a graph whose minimal layering corresponds to the classification generated by $z$ and $\rho$ and a graph with the minimal number of edges necessary to have this property.
 
 \item {\bf Step 4} The remaining edges which are not in the skeleton are added according to the scheme outlined by Kemp; for $\rho(i) < \rho(j)$, an edge probability $\xi_{ij}$ is generated according to a Beta $B(\beta_{1:\rho(i),\rho(j)},\beta_{2:\rho(i),\rho(j)})$ distribution. This is the edge probability between class $i$ nodes and class $j$ nodes when $\rho$ is the class permutation. The random variables $\xi_{i_1,j_1}$ and $\xi_{i_2,j_2}$ are independent for $(i_1,j_1) \neq (i_2, j_2)$, $\rho(i_1) < \rho(j_1)$ and $\rho(i_2) < \rho(j_2)$.  
\end{itemize}

\begin{Ex}
\label{ex:algrun2}
\end{Ex}
\noindent In \autoref{fig:partition2}-\autoref{fig:dag2} a possible outcome of the algorithm above is shown. 
The setup is similar to that in Example \ref{ex:algrun} with $d=8$,
\[ z= ( \begin{smallmatrix} 1 & 1 & 2 & 3 & 3 & 3 & 2 & 1  \end{smallmatrix} ), \quad \rho= \left ( \begin{smallmatrix} 1 & 2 & 3 \\ 1&3&2 \end{smallmatrix} \right) .\] 
The additional step in this example is step \autoref{fig:skeleton} where the skeleton is created.

\begin{figure}[!ht]
\centering


  \begin{subfigure}[b]{0.20\textwidth}
    \centering
    \begin{tikzpicture}[->,>=stealth',shorten >=1pt,auto,node distance=1.5cm,
      thick,main node/.style={circle,scale=0.7,fill=white!10,draw,font=\sffamily\Large}]
      
      \node[main node,draw=red] (1) [] {1};
      \node[main node,draw=red] (2) [right of=1] {2};
      \node[main node,draw=red] (3) [right of=2] {8};
      
       \node[main node,draw=black] (4) [below of=1] {3};
      \node[main node,draw=black] (5) [right of=4] {7};

      \node[main node,draw=blue] (6) [below of=4] {4};
      \node[main node,draw=blue] (7) [right of=6] {5};
      \node[main node,draw=blue] (8) [right of=7] {6};
    \end{tikzpicture}
    \caption{ }
    \label{fig:partition2} 
  \end{subfigure}
  ~
  \begin{subfigure}[b]{0.20\textwidth}
    \centering
    \begin{tikzpicture}[->,>=stealth',shorten >=1pt,auto,node distance=1.5cm,
      thick,main node/.style={circle,scale=0.7,fill=white!10,draw,font=\sffamily\Large}]
      \node[main node,draw=red] (1) [] {1};
      \node[main node,draw=red] (2) [right of=1] {2};
      \node[main node,draw=red] (3) [right of=2] {8};
      
      \node[main node,draw=blue] (4) [below of=1] {4};
      \node[main node,draw=blue] (5) [right of=4] {5};
      \node[main node,draw=blue] (6) [right of=5] {6};
      
      \node[main node,draw=black] (7) [below of=4] {3};
      \node[main node,draw=black] (8) [right of=7] {7};
    \end{tikzpicture}
    \caption{}
    \label{fig:cell_order2}
  \end{subfigure}
  ~
  \begin{subfigure}[b]{0.20\textwidth}
    \centering
    \begin{tikzpicture}[->,>=stealth',shorten >=1pt,auto,node distance=1.5cm,
      thick,main node/.style={circle,scale=0.7,fill=white!10,draw,font=\sffamily\Large}]
      \node[main node,draw=red] (1) [] {1};
      \node[main node,draw=red] (2) [right of=1] {2};
      \node[main node,draw=red] (3) [right of=2] {8};
      
      \node[main node,draw=blue] (4) [below of=1] {4};
      \node[main node,draw=blue] (5) [right of=4] {5};
      \node[main node,draw=blue] (6) [right of=5] {6};
      
      \node[main node,draw=black] (7) [below of=4] {3};
      \node[main node,draw=black] (8) [right of=7] {7};
      \path[every node/.style={font=\sffamily\small}]
      (1) edge node [] {} (4)
	      edge node [] {} (5)    	    
      (2) edge node [] {} (6)
      (4) edge node [] {} (8)
	      edge node [] {} (7);
    \end{tikzpicture}
    \caption{}
    \label{fig:skeleton}    
  \end{subfigure}
  ~
  \begin{subfigure}[b]{0.20\textwidth}
    \centering
    \begin{tikzpicture}[->,>=stealth',shorten >=1pt,auto,node distance=1.5cm,
      thick,main node/.style={circle,scale=0.7,fill=white!10,draw,font=\sffamily\Large}]
      \node[main node,draw=red] (1) [] {1};
      \node[main node,draw=red] (2) [right of=1] {2};
      \node[main node,draw=red] (3) [right of=2] {8};
      
      \node[main node,draw=blue] (4) [below of=1] {4};
      \node[main node,draw=blue] (5) [right of=4] {5};
      \node[main node,draw=blue] (6) [right of=5] {6};
      
      \node[main node,draw=black] (7) [below of=4] {3};
      \node[main node,draw=black] (8) [right of=7] {7};
      \path[every node/.style={font=\sffamily\small}]
      (1) edge node [] {} (4)
	      edge node [] {} (5)        
      (2) edge node [] {} (6)
    	  edge node [] {} (7)
      (4) edge node [] {} (8)
	      edge node [] {} (7)
      (6) edge node [] {} (8);
    \end{tikzpicture}
    \caption{}
    \label{fig:dag2}    
  \end{subfigure}
\caption{Figures a)-d) show the steps in the minimal Hoppe-Beta prior}
\label{fig:alg2}  
\end{figure}
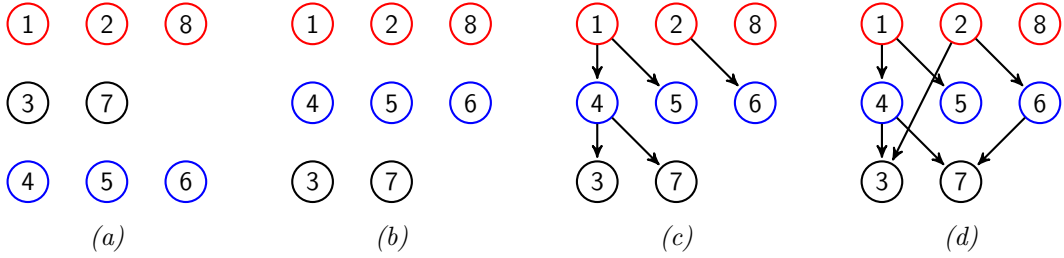
\subsection{A Formula for the Minimal Hoppe-Beta Prior} Let $G$ be a DAG on $d$ nodes whose minimal layering has $K$ layers, with $m_1, \ldots, m_K$ in each, the layers ordered from lowest to highest. 

In our scheme, a DAG implies a single class structure (allocation of objects to classes and the order of the classes). The above generation scheme gives the probability the DAG as:

\begin{eqnarray} \nonumber\mathbb{P}_{{\cal G}}(G) &=& \frac{1}{K!}\frac{\alpha^K \prod_{i=1}^K (m_i - 1)!}{\prod_{j=1}^d (\alpha  + j - 1)}\prod_{j=1}^{K-1} \frac{B(\beta_{1;j,j+1} + N_{j, j+1}- m_{j}, \beta_{2;j,j+1} + M_{j, j+1})}{B(\beta_{1;j,j+1},\beta_{2;j,j+1})}\\&& \times \prod_{j=1}^{K-2}\prod_{i=j+2}^K \frac{B(\beta_{1;j,i} + N_{j,i}, \beta_{2;j,i} + M_{j,i})}{B(\beta_{1;j,i},\beta_{2;j,i})}.\label{eqminhopbet} \end{eqnarray}

\noindent Here $\frac{1}{K!}$ is the probability of permutation $\rho$ of the classes $1, \ldots, K$, $\frac{\alpha^K \prod_{i=1}^K (m_i - 1)!}{\prod_{j=1}^d (\alpha + j - 1)}$ is the probability of the class assignment vector according to the Hoppe-Ewens urn scheme, $N_{a,b}$ denotes the number of edges between class a and class b nodes, while $M_{a,b} = m_am_b - N_{ab}$ denotes the number of missing edges. The first term follows from forcing each node of class $j$ to have at least one parent in class $j-1$; the last term follows because there is no such forcing between different pairs of classes. 

The formula follows from using

\[ \mathbb{P}(\mbox{DAG}) = \sum_{\mbox{skeleton}} \mathbb{P}(\mbox{DAG}|\mbox{skeleton})\mathbb{P}(\mbox{skeleton}).\]

\noindent For a given layering (choice of classification vector $z$ and $\rho$ - ordering of the classes; $m_j$ denotes number in class $\rho(j)$), all skeletons have the same probability $\prod_{j=2}^K \left(\frac{1}{m_{j-1}}\right)^{m_{j}}$ and all DAGs, given skeleton and layering have the same probability. 

\subsection{Properties of the Distribution} Having declared that we are restricting ourselves to the {\em minimal} layering, so that our prior is purely a prior over graph structures (and the graph structure implies the minimal layering), we can now proceed to present some straightforward properties of the distribution.

Firstly, the probability of the empty graph is clearly equal to the probability that there is exactly one class. This is:

\[ \mathbb{P}_\alpha (\mbox{empty graph}) = \frac{(d-1)!}{\prod_{j=2}^d (j-1 + \alpha)}.\]

\noindent Clearly, as $\alpha \rightarrow 0$, $\mathbb{P}_\alpha(\mbox{empty graph}) \rightarrow 1$.

\paragraph{The case of $\beta_{i;a,b} = \beta_i$ for all $a < b$, $i = 1,2$} When $\beta_{i;a,b} = \beta_i$, $i = 1,2$ for all $a < b$, it is possible to compute some reasonably straightforward properties of the prior.  In this case, the prior is a function of three parameters; $\alpha, \beta_1$ and $\beta_2$. One convenient measure of sparsity is to consider the expected number of edges and compare it to the total number of possible edges for a DAG on $d$ nodes, which is  $\frac{1}{2}d(d-1)$. 

\begin{eqnarray*} \mathbb{E}[\mbox{edges}] &=& \left(\frac{\beta_2}{\beta_1 + \beta_2}\right)(d - \mathbb{E} \left [ m_{\rho(1)} \right ]) + \left(\frac{\beta_1}{\beta_1 + \beta_2}\right) \mathbb{E} \left [\sum_{i=1}^{K-1}\sum_{j=i+1}^K m_{\rho(i)}m_{\rho(j)} \right ]\\
 &=&  \left(\frac{\beta_2}{\beta_1 + \beta_2}\right)(d - \mathbb{E} \left [ m_{\rho(1)} \right ]) + \frac{1}{2}\left(\frac{\beta_1}{\beta_1 + \beta_2}\right)\left(d^2 - \mathbb{E} \left [\sum_{i=1}^{K}  m_i^2\right ]\right).
\end{eqnarray*}

\noindent This comes from the following: firstly, each of the $d$ nodes has at least one parent (a compelled edge) for the skeleton, except for those in the lowest layer. Then there are the remaining edges. For $\rho(j) = \rho(i) + 1$, consider the non-compelled edges which are added in, each with probability $\frac{\beta_1}{\beta_1 + \beta_2}$. For the second line, only nodes within the same block cannot have an edge.   Recall that $\mathbb{E}[K] \simeq \alpha \ln (d)$. Now suppose that $\alpha(d)$ varies with $d$ and that $\alpha^\prime := \alpha(d) \ln(d)$ is kept constant. For fixed $\alpha^\prime := \alpha(d) \ln (d) > 0$, can be shown that $\lim_{d \rightarrow +\infty} \frac{1}{d^2} \mathbb{E} \left [ \sum_{i=1}^K m_i^2 \right ] = f(\alpha^\prime)$ where $f$ is a decreasing function defined on $\mathbb{R}_+$ which satisfies $f(0) = 1$ and $f(+\infty) = 0$. It follows that 

\[ \lim_{d \rightarrow +\infty} \frac{\mathbb{E} \left [ \mbox{edges} \right ]}{\frac{1}{2}d(d-1)} = \left(\frac{\beta_1}{\beta_1 + \beta_2}\right)\left ( 1 - f(\alpha^\prime)\right).\]

\noindent This may be considered as a sparsity index, since $\frac{1}{2}d(d-1)$ is the maximum number of possible edges. From this expression, it is clear that there are two parameters for controlling the sparsity. Firstly, low values of  $\left(\frac{\beta_1}{\beta_1 + \beta_2}\right)$ lead to a sparse graph. Secondly, low values of $\alpha \ln (d)$ lead to a sparse graph.  

As discussed in \citeauthor{Mansinghka06structuredpriors}, the role of $\beta_1 + \beta_2$ is also of interest. While the expected number of edges only depends on $\beta_1$ and $\beta_2$ only through $\frac{\beta_1}{\beta_1 + \beta_2}$, the sum of $\beta_1 + \beta_2$ provides a {\em consistency} parameter. If $\beta_1 + \beta_2$ is small (with $\frac{\beta_2}{\beta_1 + \beta_2} = 0.5$), the outcome of the random variable $\xi_{ij}$, which has $\mbox{Beta}(\beta_1,\beta_2)$ distribution will typically take values either close to $0$ or close to $1$. This prior will therefore generate graphs which either have many edges between a chosen pair of classes or few edges between a chosen pair of classes (while the average edge probability is $0.5$). If $\beta_1 + \beta_2$ is large, the edge probabilities between chosen pairs of classes will be similar.

\section{The Posterior Distribution and Monte Carlo Methods}
\label{sec:marg_likelihood}

Now suppose that the $d$ nodes of the graph are random variables $(X_1, \ldots, X_d)$ and an $n \times d$ data matrix ${\bf x}$ containing $n$ independent instantiations of the $d$  variables under   study  (identified as nodes) is given.  

Let ${\bf X}$ denote the random matrix from which ${\bf x}$ is an observation.  The aim is to  infer a DAG along which the probability distribution factorises and the class structure, which is the minimal layering of the DAG.

\paragraph{Cooper-Herskovits Likelihood and Posterior} Our assumption is that once the graph structure is known, the class structure (or layering) gives no further information. That is:

\[ \mathbb{P}_{{\bf X}|{\cal G}, Z} = \mathbb{P}_{{\bf X}|{\cal G}}.\]

\noindent This may be evaluated explicitly and is well known as the Cooper-Herskovits likelihood, derived in \citep{Cooper1992Bayesian}, given by: 

\begin{equation}\label{eqCH}
  \mathbb{P}_{{\bf X}|{\cal G}}({\bf x}|G) =   \prod_{j=1}^d \prod_{l=1}^{q_j} \frac{\Gamma (\sum_{i=1}^{p _j} \gamma_{jil}) }{\Gamma (n(\pi^{(l)}_j) + \sum_{i=1}^{p_j} \gamma_{jil})} \prod_{i=1}^{p_j} \frac{\Gamma(\gamma_{jil}+ n(x_j^{(i)},\pi_j^{(l)}))}{\Gamma(\gamma_{jil})}. 
\end{equation}

\noindent Here $(x_j^{(1)}, \ldots, x_j^{(p_j)})$ is the state space for variable $X_j$, while $(\pi_j^{(1)}, \ldots, \pi_j^{(q_j)})$ is a listing of the possible parent configurations for variable $j$ in the Bayesian network. The parameters $(\gamma_{jil}: j = 1,\ldots, d; i = 1, \ldots, p_j; l = 1, \ldots, q_j)$ are hyperparameters that can be chosen depending on prior information. We take $\gamma_{jil} = \gamma > 0$ all equal, so that there is one free parameter $\gamma$ for the Cooper-Herskovits likelihood. 

In the algorithms described below, computational savings are made if only a small part of the graph needs to be considered. Suppose only one edge at a time is changed. Let $G_{ij} = 1$ if there is an edge $i \mapsto j$ and $0$ if there is no edge $i \mapsto j$. Let $(x_i^1, \ldots, x_i^{p_i})$ denote the state space of variable $X_i$. Let $G^-$ denote a DAG where $G_{ij} = 0$ and let $G^+$ denote the graph $G^-$ with $G_{ij}$ replaced by $G_{ij} = 1$. Assume that $G^+$ is a DAG. Let $q_{j+}$ denote the number of parent configurations for variable $j$ in $G^+$ and $q_{j-}$ the number in $G^-$. Note that $q_{j+} = p_iq_{j-}$. The ratio $\frac{\mathbb{P}_{{\cal X}|{\cal G}}({\bf x}|G^+)}{\mathbb{P}_{{\cal X}|{\cal G}}({\bf x}|G^-)}$ may be computed, from~\eqref{eqCH}, as:

\begin{equation} 
\resizebox{.90\hsize}{!}{
$\frac{\mathbb{P}_{{\cal X}|{\cal G}}({\bf x}|G^+)}{\mathbb{P}_{{\cal X}|{\cal G}}({\bf x}|G^-)}  = 
  \left(\frac{\Gamma (p_j\gamma) }{\Gamma(\gamma)^{p_j}}\right)^{q_{j-}(p_i - 1)}\label{eqCHupdate}
 \prod_{l=1}^{q_{j-}} \frac{\Gamma(p_j\gamma + n(\pi_j^{-(l)}))}{\prod_{k=1}^{p_j} \Gamma(p_j\gamma + n(\pi_j^{-(l)},x_i^{(k)}))} 
\prod_{k=1}^{p_i}\prod_{\alpha = 1}^{p_j} \frac{\Gamma(\gamma + n(\pi_j^{-(l)}, x_j^{(\alpha)}))}{\prod_{k=1}^{p_j} \Gamma(p_j\gamma + n(\pi_j^{-(l)},x_i^{(k)},x_j^{(\alpha)}))},$
}
\end{equation}
\noindent where $(\pi_j^{-(1)}, \ldots, \pi_j^{-(q_{j-})})$ is an enumeration of the parent configurations of variable $j$ in graph $G^-$. This formula only depends on the variable $X_j$, the parent set $\mbox{Pa}_j^-$ of $X_j$ in graph $G^-$ and the additional variable $X_i$. \vspace{5mm}

\noindent We tried two algorithms; a Gibbs sampler and a stochastic optimisation algorithm. These two algorithms have different objectives. The aim of a Gibbs sampler is to generate an empirical distribution which approximates the posterior distribution. This is useful for exploring properties of the posterior. The stochastic optimisation algorithm simply looks for the maximum aposteriori structure. 

While the Gibbs sampler is theoretically ergodic, convergence was very slow. The main difficulty was that nodes in the wrong layer had difficulty bubbling up to their correct layer.

With this in mind, the stochastic optimisation algorithm was constructed to ensure mobility between layering. The moves were constructed by choosing a node at random and re-assigning it according to to a Hoppe-Ewens urn scheme. 

\subsection{Gibbs Sampler}

We now describe the Gibbs sampling scheme. For the Minimal Hoppe-Beta posterior, there are $d(d-1)$ variables;   $(G_{ij})_{(i,j) \in \{1, \ldots, d\}^2, i \neq j}$. Variable $G_{ij}$ is a binary variable taking value $1$ if and only if the graph $G$ has a directed edge $i \mapsto j$. 

We consider a Gibbs sampler, working through these variables one by one, conditioning on all the other variables. Let $\underline{X} = (X_1, \ldots, X_{d(d-1)})$ represent an enumeration of the binary variables $(G_{ij})_{i \neq j}$. Denote the $i$th sample $(G^{(i)})$ by $(x^{(i)}_1, \ldots, x^{(i)}_{d(d-1)})$. The algorithm proceeds as follows: \vspace{5mm}

\noindent {\bf Initialisation} Let $(G^{(0)})$ be the initial condition, where $G^{(0)}$ is the empty graph.\vspace{5mm}
 
\noindent {\bf Sampling} To generate a random sample of size $k$, for each sample $i \in \{1, \ldots k\}$, do the following:
\begin{itemize}
 \item For  $j = 1, \ldots, d(d-1)$, sample $x^{(i)}_j$ from the conditional distribution 
\begin{equation}\label{eqsamppr} \mathbb{P}_{X_j|\underline{X}_{-j}}(.|x_1^{(i)},\ldots, x_{j-1}^{(i)},x_{j+1}^{(i-1)}, \ldots, x_{d(d-1)}^{(i-1)}) =:\alpha_j
\end{equation}
\end{itemize}

\noindent where $\underline{X}_{-j} = (X_1, \ldots, X_{j-1},X_{j+1}, \ldots, X_{d(d-1)})$.\vspace{5mm}

\noindent The probability $\mathbb{P}_{{\cal G}}$ from~\eqref{eqminhopbet} together with the Cooper-Herskovits likelihood gives the posterior:

\begin{equation}
  \mathbb{P}_{{\cal G}|{\bf X}}(G|{\bf x}) \propto    \mathbb{P}_{{\cal G}}(G) \mathbb{P}_{{\bf X}|{\cal G}}({\bf x}|G) =: \mathbb{S}(G|{\bf x})
  \label{eqdzjoint2}
\end{equation}

 \noindent If $X_j$ to be sampled is a graph edge variable $G_{xy}$ then it takes value $0$ with probability $1$ if the graph with an edge $x \mapsto y$ is illegal (i.e. leads to a directed loop). Otherwise, $\alpha_j(G_{xy} = 1)$ is computed from Equations~\eqref{eqminhopbet} and~\eqref{eqCHupdate}.
\subsection{Stochastic search algorithm} 
\label{sec:stochastic_search_algorithm}
In this section we present the stochastic search algorithm used for maximising, at least approximately, the posterior score function defined in Equation~\eqref{eqdzjoint2}.
The algorithm generates a non-reversible Markov chain of DAGs $\{G^{(t)}\}_{t=1}^T$, started from an initial DAG $G^{(0)}$.
A move from site $G^{(t)}$ to site $G^{(t+1)}$ is then made by proposing a move to $G^\prime$ according to the distribution $Q(G^{(t)};.)$ where the transition kernel $Q$ is described below and then accepting the move with probability

\begin{align}
\alpha_{G^{(t)},G^\prime} = 
\min \bigg \{1\,,\, \frac{\mathbb S(G^\prime| \mathbf x) }{\mathbb S(G^{(t)}| \mathbf x) }\bigg \}.
\end{align}

\paragraph{The Proposal Kernel Q}
For each DAG $G^{(t)}$ in the chain we let $\underline z^{(t)}$ denote the corresponding layering.
The algorithm is initiated with the empty graph $G^{(0)}$, thus $\underline z^{(0)}$, the minimal layering of $G^{(0)}$ is a vector of zeros.
We proceed as follows:
\begin{enumerate}
\item 
Let  $\widetilde{z} =z^{(t)}$, the minimal layering of $G^{(t)}$ and call $\widetilde{z}_{\mbox{old}} = \widetilde{z}$.
Choose a node $x$ at random and remove it (each node with equal probability).
If the chosen node   was in a layer of its own, 
decrease the order of each layer with order $\geq j$, that is set  $\widetilde{z}_{\mbox{new}} = \widetilde{z}_{\mbox{old}}-1$ for all $\widetilde{z}_k$ such that $\widetilde{z}_k \ge j$ and then set $\widetilde{z} = \widetilde{z}_{\mbox{new}}$ 
so that there will be  no  empty  layers. Set $\widetilde{z} := \widetilde{z}_{\mbox{new}}$. 

\item 
Let $K$ be the number of layers left. 
Now re-assign the node $x$  in the following way:  put it in a new layer, ($K+1$) with probability $\frac{\widetilde{\alpha}}{\widetilde{\alpha} +(d-1)}$ and in layer $j$ for $1\le j \le K$ with probability $\frac{m_j}{\widetilde{\alpha} + (d-1)}$ where $m_j$ is the   current occupancy number   in layer $j$. In other words, the node is reassigned according to a Hoppe-Ewens scheme with parameter $\widetilde{\alpha}$.

\item If $x$ is now in layer $K+1$, let $m$ be chosen randomly according to the uniform distribution over $\{1,\dots,K+1\}$ and set: 
 
\begin{align}
    \begin{cases}
      \hat{z}_{j} = \widetilde{z}_j, & j: 1\le \widetilde{z}_j\le m-1 \\
      \hat{z}_{i} = m, & j:j=i\\ 
      \hat{z}_{j}= \widetilde{z}_j+1, & j:  m+1\le \widetilde{z}_j\le K 
    \end{cases} 
\end{align}
That is, the new layer is placed in position $m$ and the indices from $m$ to $K$ are pushed one place to the right to 
compensate.  
This process has generated a new partition vector $\hat{z}$.

\item
To construct a proposal $G^\prime$, let $\widetilde{G} = G^{(t)}$. The minimal layering, $z^\prime$ for $G^\prime$ will be derived from $\hat{z}$ constructed above.  

\begin{enumerate}
\item Remove those edges contradicting the partition $\hat{z}$ (i.e {\em from} class $a$ {\em to} class $b$ for $b \leq a$).
\item If $\hat{z}_x \geq 2$, if $x$ does not have any parent in layer $\hat{z}_x-1$, add a single compelled edge $(p,x)$ where $p$ is randomly chosen from layer $\hat{z}_x - 1$, each with probability $\frac{1}{\hat{z}_x-1}$. 
\item If $\hat{z}_x=1$, add $x$ as a parent to each node in $\hat{z}_x+1$.
\item Steps (a), (b), (c) have generated a DAG $\widehat{G}$. Let $\hat{z}_{\mbox{new}}$ denote the layering of DAG $\widehat{G}$. This differs from $\hat{z}$ only if there are children of $x$ in graph $\widetilde{G}$ for which $x$ was the {\em only} parent in $\widetilde{G}$. For such nodes, $\hat{z}_{\mbox{new},c} = \max_{p \in \mbox{Pa}_{\widetilde{G}}(c)\backslash\{x\}} \widetilde{z}_p + 1$. Here $\mbox{Pa}_{\widetilde{G}}(c)\backslash\{x\}$ denotes the parent set of $c$ in graph $\widetilde{G}$ without $x$. This process continues recursively for the children of $x$ until every node in the graph is in it minimal layer.
\end{enumerate}

\item Let $\hat{z}$ denote the current partition vector ($\hat{z}$ from step 3 modified by step 4 (d)). Now choose two nodes $y$ and $w$ at random. If  $\hat{z}_y \geq \hat{z}_w$, do nothing. Otherwise, toggle the edge between $y$ and $w$ (i.e remove it if exists and add it if it does not exists). Let $G^\prime$ denote the resulting graph. 
\item Let $z^\prime$ be the minimal layering of graph $G^\prime$. If $(y,w)$ was removed, then place $w$  in its minimal layer as in step 4 (d).
\end{enumerate}

\noindent This process generates a proposal DAG $G^\prime$ with minimal layering $z^\prime$.\vspace{5mm}

\noindent Alternatively, the layering $\hat{z}$ from step 3 could have been taken as the minimal layering of the graph, with step 4(d) replaced by adding in a minimal number of edges to ensure that children $c$ of node $x$ were in the appropriate class. The rejection rate was higher with this approach.

\section{Simulations} 
\label{sec:simulations}
\subsection{Data generation} 
\label{sub:data_generation}

The simulation studies were made on random samples from the HEPAR II network shown in \autoref{fig:heparII}.
All the variables are binary and the parameter in the conditional probability tables where independently sampled from a $Beta(0.5, 0.5)$- distribution and were then adjusted so to ensure that they lie in the range $(0.1,0.9)$.

\subsection{Simulation results} 
\label{sub:simulation_results}
The stochastic search algorithm with the three types of priors (uniform, Hoppe-Beta and minimal Hoppe-Beta) was tested on 500  samples from the HEPAR II network shown as an adjacency matrix in \autoref{fig:heparII_adjmat} \footnote{The same study were performed on 10 different datasets showing similar results, for that reason, results from only one representative dataset are considered here.}.
For the minimal Hoppe-Beta and the Hoppe-Beta prior, the hyper parameters were set to $\beta_{1;i,i+1} = 2$, $\beta_{2;i,i+1} = 1$ for $i=1,\dots,K-1$ and $\beta_{1;i,j} = 1$, $\beta_{2;i,j} = 2,$ for $i=1,\dots,K-2$ and $\, K\geq j>i+1$. 
This reflects the prior knowledge that between layers $i$ and  $i+1$, the graph is denser, while between layer $i$ and $j$ where $j>i+1$ the graph is sparser.
For the sparsity parameter we used $\alpha=1$.
For all the three priors we used $\gamma=1$ and $\widetilde \alpha = 1$ .

As a measure of goodness of fit 
we use the sensitivity (TPR) and specificity (SPC) defined as follows
\begin{align*}
  TPR \stackrel{\rm def}{=} \frac{\text{\emph{Number of edges correctly identified}}}{\text{\emph{Number of edges correctly identified + Number of edges falsely rejected}}}
\end{align*}
and
\begin{align*}
  SPC \stackrel{\rm def}{=} \frac{\text{\emph{Total number of edges in the skeleton}}}{\text{\emph{Total number of edges in the skeleton + Number of edges wrongly included}}}.
\end{align*}
Here, the skeleton of a directed network means its undirected version.

We ran 10 trajectories for 5000 iterations for each prior and the results are summarized in \autoref{tab:results} and in \autoref{fig:stochastic_search_hoppe}-\autoref{fig:stochastic_search_hoppe_beta}.
Since the Hoppe-Beta prior is joint prior over partitions and DAGs, we used this joint score when evaluation this method.
The mean and standard error for the TPR and SPC taken over the 100 best scoring graphs among the 10 chains are found in~\autoref{tab:results}.
As seen the minimal Hoppe-Beta prior shows the best results in terms of SPC and TPR among the methods.

Read from the top, \autoref{fig:stochastic_search_hoppe} - \autoref{fig:stochastic_search_hoppe_beta} show in the left columns, the trajectories for the prior, the Cooper-Herskovits likelihood and the scoring function. 
The black line in each plot is the corresponding function evaluated at the true HEPAR II network.
The right column shows heat maps over: the top 100 scoring DAGs among all 10 trajectories, the single top scoring DAG among the 10 trajectories and the top scoring graph in each trajectory.
Looking at the heat maps for the top 100 scoring graphs from the 10 chains, we see that we get the clearest results with the minimal Hoppe-Beta prior.

\begin{table}[h]
\begin{center}
\begin{tabular}{|l|l|l|}
\hline 
Prior &  TPR  & SPC   \\
\hline
Uniform & 0.77 / 0.09 &  0.79 / 0.07 \\
Minimal Hoppe-Beta & 0.85 / 0.10  &  0.86 / 0.06\\
Hoppe-Beta & 0.80 / 0.19 & 0.84 / 0.12 \\
\hline
\end{tabular}
\end{center}
\caption[Simulation results]{The results reported are mean values and standard error (mean/std. error) taken over the 100 best scoring DAGs among the 10 trajectories of the stochastic search algorithm ran for 5000 iterations each.}
\label{tab:results}
\end{table}
\begin{figure}[H]
	\centering
	\includegraphics[width=0.40\textwidth]{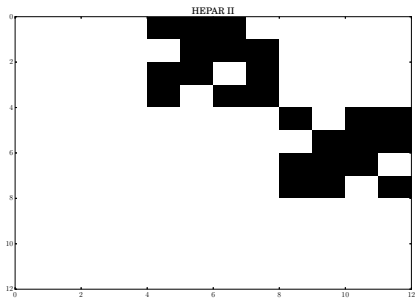}
	\caption{The HEPAR II network shown as an adjacency matrix.}
	\label{fig:heparII_adjmat}
\end{figure}
\begin{figure}[H]
	\centering
	\includegraphics[width=0.95\textwidth]{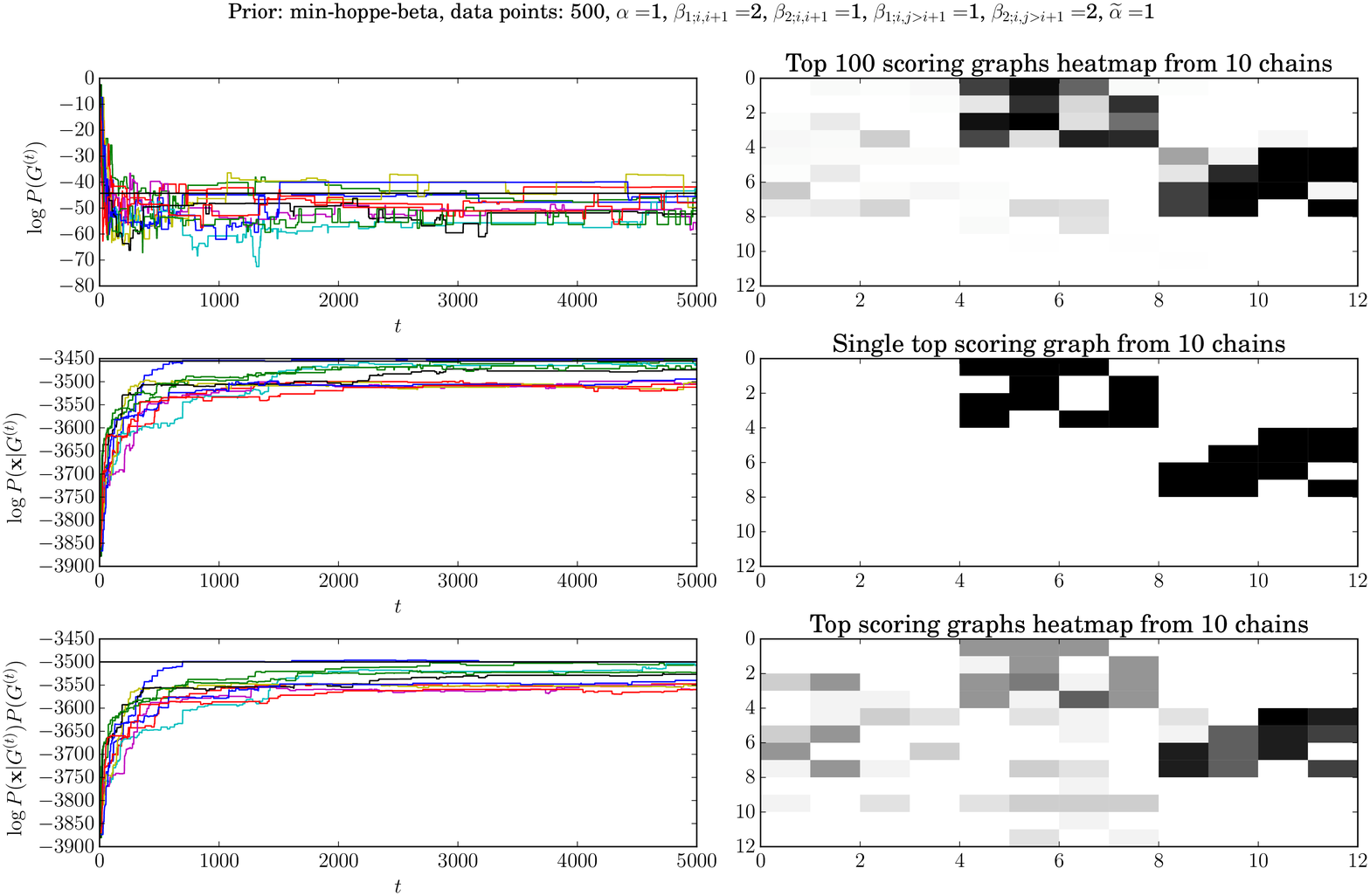}
	\caption{Results from 10 stochastic search trajectories with the minimal Beta-Hoppe prior ran for 5000 iterations each.}
	\label{fig:stochastic_search_hoppe}
\end{figure}

\begin{figure}[H]
	\centering
	\includegraphics[width=0.95\textwidth]{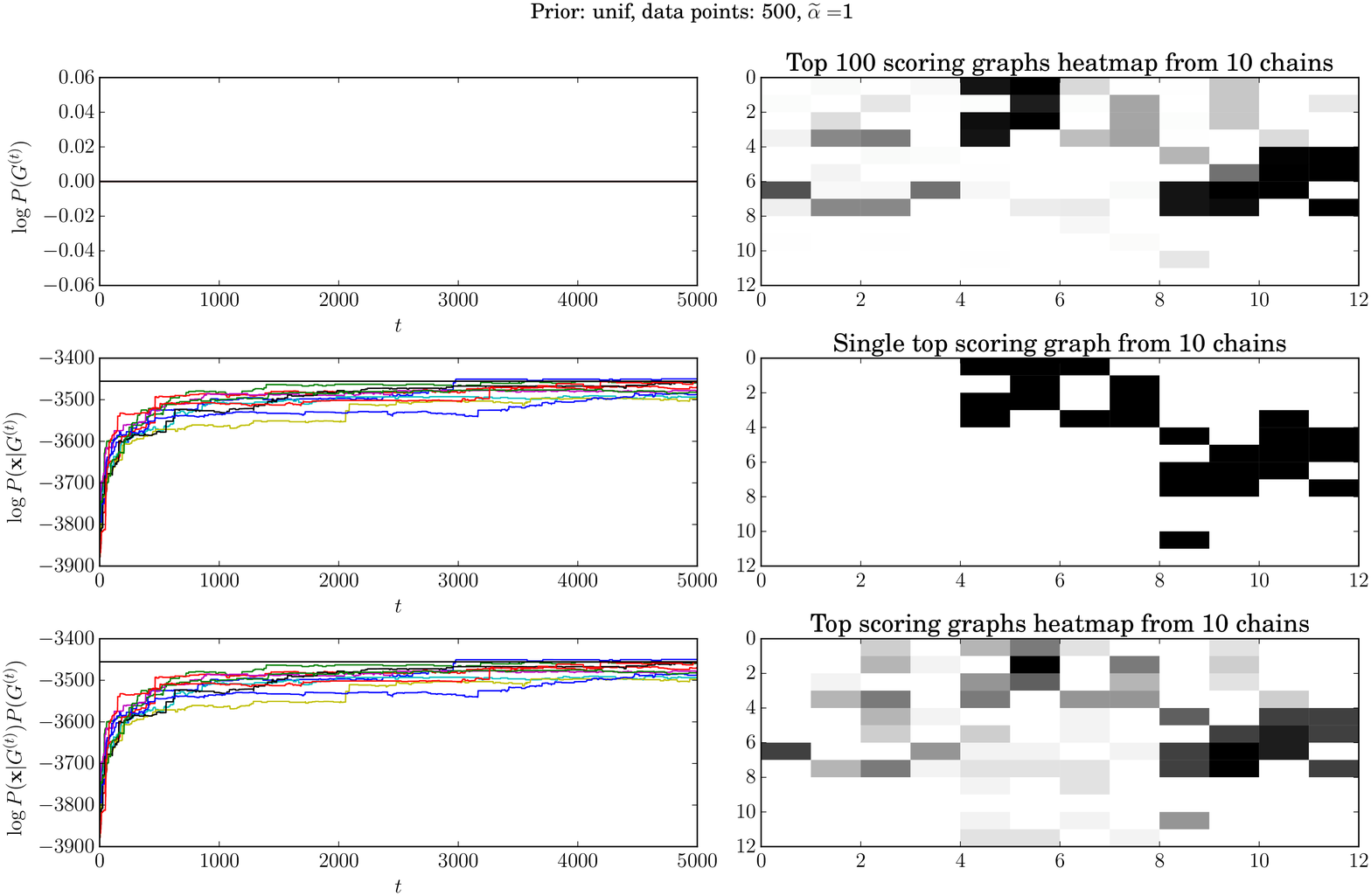}
	\caption{Results from 10 stochastic search trajectories with the uniform prior ran for 5000 iterations each.}
	\label{fig:stochastic_search_unif}
\end{figure}

\begin{figure}[H]
	\centering
	\includegraphics[width=0.95\textwidth]{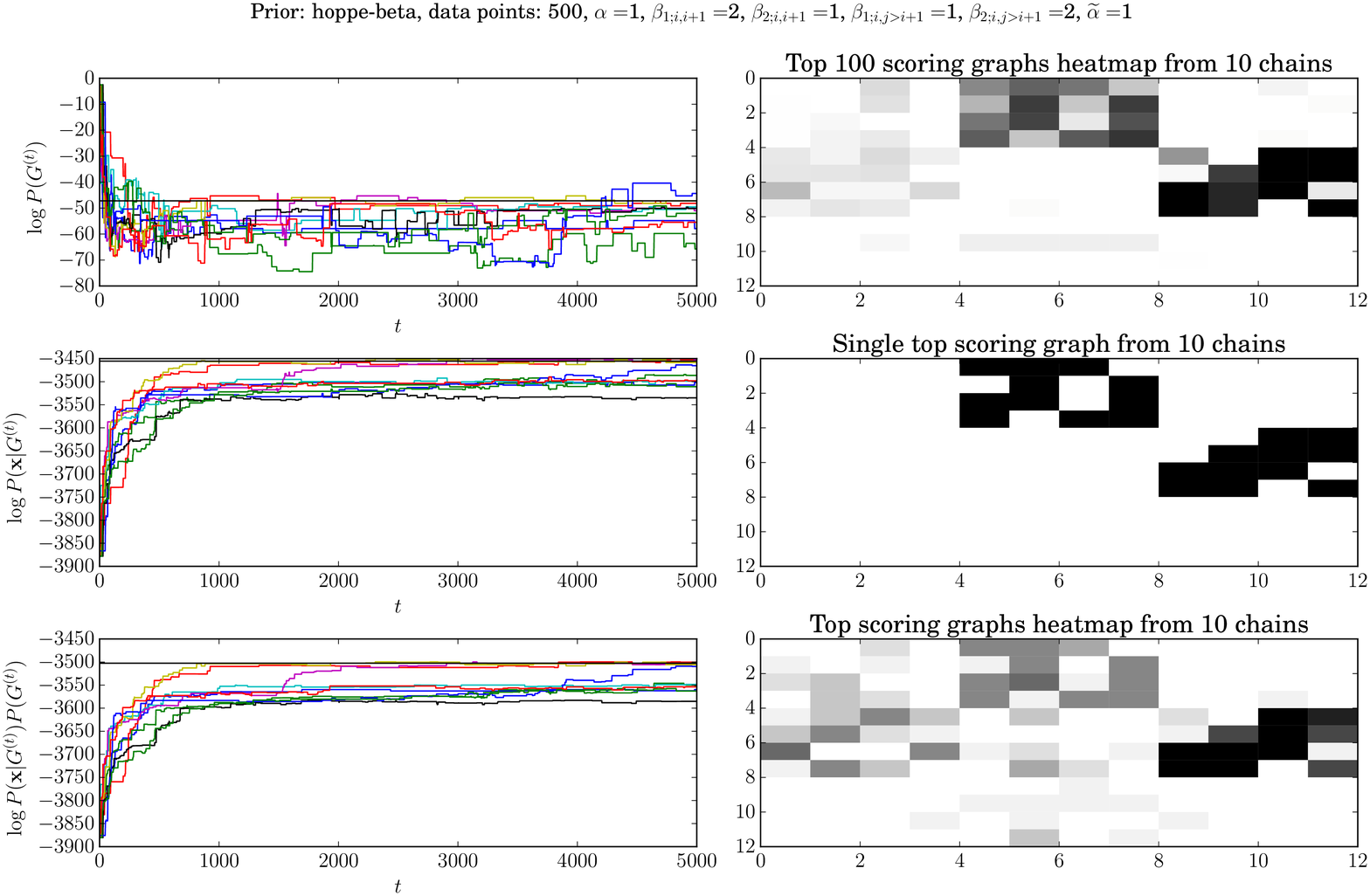}
	\caption{Results from 10 stochastic search trajectories with the Hoppe-Beta prior ran for 5000 iterations each.}
	\label{fig:stochastic_search_hoppe_beta}
\end{figure}


\section{Summary and Conclusions} 
\label{sec:summary_and_conclusions}
This article considered the {\em Ordered Block Model} introduced by \citet{Kemp2004}, which was applied to Bayesian Networks for multivariate data by  \citet{Mansinghka06structuredpriors}. We introduce a new prior distribution over graph structures, which is a modification of the Hoppe-Beta prior introduced by \citeauthor{Kemp2004}. The probability distribution over graphs has an explicit closed form. The parameters can adjusted to determine {\em sparsity} and {\em consistency}; consistency refers to the similarity of edge probabilities between nodes of different pairs of classes. We call this prior the {\em minimal Hoppe-Beta}. It builds on, and represents an advantage over the Hoppe-Beta of \citeauthor{Kemp2004}, which is a joint prior over graphs and classes. With the minimal Hoppe-Beta, the class structure is implied by the graph.

The prior is then tested experimentally; a posterior is obtained via the Cooper-Herskovits likelihood. We run a stochastic optimisation scheme and compare the output with three different priors; the uniform, the Hoppe-Beta of \citeauthor{Kemp2004} and the new Minimal Hoppe-Beta. 

The Minimal Hoppe-Beta compares favourably.


\setlength{\baselineskip}{2ex}
\bibliographystyle{plainnat.bst}
\bibliography{allbib,Remote}

\end{document}